%% file: main.tex
\documentclass[11pt]{article}
\usepackage{gal-basic}

\title{Who's responsible? Jointly quantifying the contribution of the learning algorithm and  data
}
\author{Gal Yona\thanks{}\\Weizmann Institute of Science\\\texttt{gal.yona@weizmann.ac.il}
\and
Amirata Ghorbani\thanks{}\\Stanford University\\\texttt{gal.yona@weizmann.ac.il}
\and
James Zou\thanks{}\\Stanford University\\\texttt{gal.yona@weizmann.ac.il}}
\author{Gal Yona \\ Weizmann Institute of Science
\and
Amirata Ghorbani \qquad James Zou\thanks{Correspondence: \texttt{jamesz@stanford.edu}}\\Stanford University}

\begin{document}

\maketitle

\input{abstract}

\section{Introduction}
\label{sec:intro}

 As machine learning is deployed in consequential domains, questions of responsibility and accountability in ML are becoming increasingly important. 
One interesting example is the debate around a recent super-resolution algorithm \cite{menon2020pulse}. Users that experimented with this algorithm for the purpose of generating high-quality images of people from their pixelated versions observed that the performance for non-white individuals was poor. For example, when queried with a pixelated image of president Obama, the model returned a photo of a white man. ``Other generated images gave Samuel L. Jackson blonde hair, turned Muhammad Ali into a White man, and assigned White features to Asian women'' \cite{VentureBeatAIBias}. 

These findings sparked internet controversy around the source of these failures (see \cite{TheGradientPulseLessons} for a comprehensive overview). Some have argued that since the training data used for the algorithm is itself ``biased'' (does not have equal representation for white and non-white individuals), surely the results will be biased too. Others have argued that a dataset will never be perfectly balanced, and this cannot serve as an excuse for the deterioration of performance. They hypothesized that a different learning objective would have produced far better results. 

The debate highlights that reasoning about  questions of ML accountability can be challenging. The first perspective can be considered data-centric, in that it views unfairness of the eventual ML model as a property of  the training data alone. The second perspective is more holistic as it highlights that there are many other choices in the algorithm design process that may share some of the responsibility. However, the typical way in ML to compare two algorithms is to evaluate the difference in their performance on some fixed benchmark. This can be viewed as algo-centric, in the sense that it ignores the potential effect the training data has on this gap. Both of these perspectives are therefore limited in their ability to address real-world questions of accountability, in which there are  complex interactions between the training data and the different algorithmic choices made when developing the model.

\subsection{Our contributions} 

Motivated by such scenarios and these two opposing perspectives (data-centric and algo-centric), we provide the first rigorous formulation of the problem of joint credit assignment between the learning algorithm and training data. Despite its' importance in developing accountable ML systems, we find that this question has not been  formulated or systematically studied.

We then propose Extended Shapley as a principled framework for the joint credit assignment problem. We prove that it is the unique valuation scheme satisfying several desirable properties, and characterize the resulting values from the perspective of both the datapoints and the algorithm (Theorem \ref{prop:unique_p1p5}). As in general the values are expensive to compute, we extend  approximation schemes introduced in ~\cite{ghorbani2019data} to efficiently compute Extended Shapley even in very large datasets.

Finally, we empirically demonstrate how Extended Shapley can be used to address  questions pertaining to accountability. Viewing Extended Shapley as a novel performance metric, our focus is on comparing between the insights provided by the value Extended Shapley assigns to the algorithm (versus some benchmark algorithm), as opposed to simply comparing the marginal difference in performance between the two. We hope that this could pave the way for Extended Shapley (and potentially other approaches to the joint credit assignment problem) to be used to debug realistic failure cases such as the super-resolution example  discussed above.

\subsubsection{Related work}
Shapley value was proposed in a classic paper in game theory \cite{shapley1953value}. It has been applied to analyze and model diverse problems including voting and bargaining \cite{milnor1978values, gul1989bargaining}. Recent works have adopted Shapley values to quantify the contribution of data in ML tasks~\cite{ghorbani2019data,jia2019towards, agarwal2018marketplace,ghorbani2020distributional}. However in these settings, the ML algorithm is assumed to be given and the focus is purely on the data value.     Shapley values have also been used to analyze algorithm portfolios \cite{frechette2016using} with extensions to temporal setting found in \cite{kotthoff2018quantifying}. But in these settings the algorithms are not ML algorithms and the contribution of the data is not modeled.     In contrast, in our extended Shapley framework, we jointly model the contribution of both the learning algorithm and the data. 

In other ML contexts, Shapley value has  been used as a feature importance score to interpret black-box predictive models \cite{owen2014sobol,kononenko2010shap1,datta2016shap2,lundberg2017shap3, cohen2007shap4, ghorbani2019interpretation, chen2018shapley, lundberg2018consistent}. Their goal is to quantify, for a given prediction, which features are the most influential for the model output. Our goal is very different in that we aim to quantify the value of individual data points (not features). 
There is also a literature in estimating Shapley value using Monte Carlo methods and analytically solving Shapley value in specialized settings \cite{fatima2008linear, michalak2013efficient, castro2009approx1, maleki2013bounding, hamers2016new}. 
Recent works have studied Shapley value in the context of data valuation focusing on approximation methods and applications in a data market.~\cite{ghorbani2019data, jia2019towards, agarwal2018marketplace}.

Finally, algorithmic accountability is an important step towards making ML more fair and transparent~\cite{zou2018ai}, yet there has been relatively little work on quantitative metrics of accountability. In the context of fairness, \cite{hooker2021moving} argues that fairness shouldn't be viewed as solely a data problem since seemingly subtle model design choices may amplify existing biases in the data \cite{hooker2020characterising, hooker2019compressed, bagdasaryan2019differential}. To the best of our knowledge, ours is the first work to formalize and study the problem of jointly 
quantifying the contribution of both training data and the learning algorithm in ML.   

\section{Jointly modeling the contribution of data and algorithm}

\subsection{Data valuation} \label{sec:prelims}
Recent works studied how to quantify the value of the individual training datum for a fixed ML model~\cite{ghorbani2019data}. In the data valuation setting, we are given a training set $N = \{z_1, ..., z_n\}$ of $n$ data points and we have a fixed learning algorithm $A$\footnote{We will often denote $z_i$ by its index $i$ to simplify notation.}. $A$ operates on any subset $S \subseteq N$, producing a predictor $A(S)$. The performance of $A(S)$ is quantified by a particular metric---accuracy, $F_1$, etc.---evaluated on the test data, and we denote this as $v_A(S)$. More formally, $G \equiv \{v:2^{N}\to\R$ \mbox{ such that } $v(\emptyset)=0$\} denote the space of performance functions and $v_A \in G$. The algorithm's overall performance is $v_A(N)$ and the goal of data valuation is to partition $v_A(N)$ among $i \in N$. 

Drawing on the classic Shapley value in cooperative game theory \cite{shapley1953value,shapley1988shapley}, it turns out that that there is a unique data valuation scheme, $\phi: G \to \R^n$, that satisfies four reasonable equitability principles:
\begin{enumerate}[S1.]
    \item Null player: If $i\in N$ is a null player in $v$---i.e. $v(S+i)=v(S)$ $\forall S \subseteq N-i$---then $\phi_i(v)=0$.
    
    \item Linearity: If $v_1, v_2 \in G$ are two performance functions, then $\phi(v_1+v_2)=\phi(v_1)+\phi(v_2)$.
    
    \item Efficiency: $\sum_{i\in N}\phi_{i}(v)=v(N)$.
    
    \item Symmetry: If $i,j \in N$ are identical in $v$---i.e.  $v(S+i) = v(S+j)$ $\forall S \subseteq N-\{i,j\}$---then $\phi_i(v)=\phi_j(v)$.
\end{enumerate}

This unique data Shapley value  \cite{ghorbani2019data} is given by:

\begin{equation}
\label{eqn:shapley}
\phi_i(v)= \sum_{S\subset N-i} {\tau_{S,N}\cdot \sbr{v(S +i ) - v(S)}}
\end{equation}

where $\phi_i$ denotes the value of the $i$-th training point and $\tau_{S,N} = \frac{\card{S}!(\card{N}-\card{S}-1)!}{\card{N}!}$ is a weighting term.

\subsection{Data and learning algorithm valuation}
While data Shapley (Eqn.~\ref{eqn:shapley}) is useful---e.g. it identifies poor quality data and can improve active learning \cite{ghorbani2019data}---it implicitly assumes that \emph{all} of the performance is due to the training data, since $\sum_{i \in N} \phi_i = v_A(N)$. This data-centric view is limiting, since we know that a good design of the learning algorithm $A$ can greatly improve performance, and hence $A$ deserves some credit. Conversely, if $A$ fails on a test domain, $A$ should share some of the responsibility. Much of ML takes the other extreme---the algorithm-centric perspective---in assigning credit/responsibility. Many ML papers present a new learning algorithm $A$ by showing a performance gap over a baseline algorithm $B$ on a standard dataset $N$. It is often framed or implicitly assumed that the entire performance gap is due to the higher quality of $A$ over $B$. As we will discuss, this algorithm centric perspective is also limiting, since it completely ignores the contribution of the training data in $N$ to this performance gap. 

 In many real-world applications, we would like to jointly model the value of the data and learning algorithm. Training data takes time and resource to collect and label. Developing an advanced algorithm tailored for a particular dataset or task also requires time and resource, and it can lead to a better outcome than just using the off-the-shelf-method (e.g. a carefully architectured deep net vs. running scikit-learn SVM~\cite{scikit-learn}). Moreover, in many applications, the algorithm is adaptively designed based on the data.
 Therefore it makes sense to quantify the contribution of the data and algorithm together, which is the premise of Extended Shapley. To make our approach more flexible we include a baseline model $B$. This only allows us more flexibility in general and is not a constraint: when we are interested in the value of $A$ without comparing it to a baseline we can always set $v_B =0$. In some applications, there's a natural baseline and it's useful to specify that as $B$.

\paragraph{Warm-up: Symmetric version. } %To gain intuition as we build up to our proposed solution, consider the following 
A naive extension of Data Shapley to this new setup is to explicitly add the algorithm $A$ and benchmark $B$ to the coalition as the $n+1$-th and $n+2$ synthetic ``data''. 
%we consider two possible approaches for extending the Data Shapley framework to this new setup:
%\paragraph{Naive extension.} One straightforward approach for accounting for the value of $\L$ is to explicitly add the learner to the coalition; i.e., 
Then use the Data Shapley as above (Equation \ref{eqn:shapley}) with respect to the enlarged coalition and the following modified value function for $S \subseteq N+2$:
\vspace{-0.5cm}

\begin{equation}
\label{eqn:naive_model}
v'(S)=
\begin{cases}
0 & A,B\notin S\\
v_{A}(S - \set{A}) & A \in S, B\notin S \\
v_{B}(S - \set{B}) & A \notin S, B\in S \\
\max\{v_A(S - \set{A,B}), v_B(S - \set{A,B})\} & A \in S, B\in S
\end{cases}
\end{equation}
where $v_A(S)$ and $v_B(S)$ are the performance of $A,B$ trained on $S$, respectively. However, this  model is not appropriate to settings in which the baseline $B$ is an existing algorithm and the objective is to quantify the improvement of $A$ over the baseline model. For example, 
suppose the ML developer is lazy and provides a fancy algorithm $A$ which is exactly the same as $B$ under-the-hood. Then $A$ and $B$ would be completely symmetric in Eqn.~\ref{eqn:naive_model}, and would receive the same value (which could be substantial) under Eqn.~\ref{eqn:shapley}. Taking the perspective that $B$ is ``common knowledge'', under this scheme the ML developer  receives credit for doing  no work, which is undesirable. 

%We argue, however, that this doesn't completely capture our objective. Suppose . 
%Recalling the neural network vs. SVM example from above: if the learner simply used the off-the-shelf model $B$ (requiring zero effort on their part), this approach would still assign them value! This highlights that in order to properly reason about the value created by the learner, we must explicitly take into the performance of the learner's chosen algorithm relative to some benchmark algorithm (in the example above, the off-the-shelf model $B$).

\section{Extended Shapley}
Building on the notation of Section~\ref{sec:prelims}, we now formally define our framework.  We focus on the asymmetric version, in which an ML designer  develops algorithm $A$, and $B$ denotes an existing baseline ML model. For. example, $A$ could be a new tailored network and $B$ an off-the-shelf SVM. Both $A$ and $B$ are trained on the dataset $N$. Our goal is to jointly quantify the value of each datum in $N$ as well as the ML designer who proposes $A$.  We are not interested in the value of $B$; its inclusion makes the framework more flexible and, as mentioned previously,  can always be set to zero.

Let $v_{A},v_{B}\in G$ denote the performances of $A$ and $B$, respectively. Denote $v=(v_{A},v_{B}) \in G\times G \equiv G^2$. We are now looking for an extended valuation function, $\varphi:G^2 \to\R^{n+1}$. Note that $\varphi$ is defined over pairs of games, and that its' range is in $\R^{n+1}$. We interpret $\varphi_{i}(v)$ for $i=1,\dots,n$ as the value for a datapoint $i\in N$ and $\varphi_{n+1}(v)$, sometimes denoted $\varphi_A(v)$,  as the payoff for the ML designer who develops $A$.

Naturally, there are infinitely many possible valuation functions. Following the approach of the original Shapley value, we take an axiomatic approach by first laying out a set of reasonable properties that we would like an equitable valuation to satisfy:

\begin{enumerate}[P1.]
    \item Extended Null Player: If $i\in N$ is a null player in both $v_{A},v_{B}$, then $\varphi_{i}(v)=0$. Additionally, if $A$ is identical to the benchmark $B$---i.e. $v_A(S) = v_B(S)$ $\forall S \subseteq N$---the designer should receive $\varphi_{n+1}(v)=0$. 
    
    \item Linearity: Let $v_{A1}, v_{A2}, v_{B1}, v_{B2} \in G$ be any four performance functions. Consider $v_1 = (v_{A1}, v_{B1})$ and $v_2 = (v_{A2}, v_{B2})$. Then, $\varphi(v_1+v_2) = \varphi(v_1) + \varphi(v_2)$.
    
    \item Efficiency w.r.t $v_{A}$: $\varphi_A(v) + \sum_{i\in N}\varphi_{i}(v)=v_{A}(N)$.
    
    \item Symmetry between data: If $i,j \in N$ are identical in both $v_A,v_B$, then $\varphi_i(v)=\varphi_j(v)$.
    
    \item Equitability between data and algorithm: if adding $i \in N$ and adding $A$ have the same effect---i.e. $\forall S \subseteq N-i$, $v_B(S+i) = V_A(S)$---then $\varphi_i(v)=\varphi_A(v)$.
\end{enumerate}

P1 to P4 are direct analogues of the fundamental fairness axioms of the Shapley value and the data Shapley value. Note that P1 encodes the asymmetry between $A$ and $B$. P5 is also a reasonable property that ensures that the algorithm has no ``special'' role: if the algorithm and data have exactly the same effect, they receive the same value. 
It might be helpful to think of the extreme case where $A$ is simply adding another datum $z_A$ to the training set and then apply $B$. In this case, P5 is equivalent to P4.

The next proposition proves that there is a  unique valuation scheme that satisfies properties P1-P5, and characterizes the values  $A$ and the data receive under this scheme. 

\begin{proposition}
\label{prop:unique_p1p5}
There is a unique valuation $\varphi: G^2 \to \R^{n+1}$ that satisfies $P1$ to $P5$,  given by
\begin{equation}\label{eqn:algorithm_value}
    \varphi_A(v)=\sum_{S \subseteq N} \frac{S! (N-S)!}{(N+1)!}[v_A(S) - v_B(S)]
\end{equation}
\vspace{-0.5em}
\begin{multline}\label{eqn:data_value}
    \varphi_i(v) = \sum_{S\subseteq  N-i}{w_{S,N}\cdot\tau_{S,N}\cdot  \sbr{v_A(S+i)-v_A(S)}}   +  \sum_{S\subseteq  N-i}{\br{1-w_{S,N}}\tau_{S,N}\cdot  \sbr{v_B(S+i)-v_B(S)}}
\end{multline}
where $w_{a,b}=\frac{a+1}{b+1}$. We call $\varphi_A, \varphi_i$ the Extended Shapley values.

\end{proposition}

Please see Appendix \ref{app:proofs:es} for the proof.

\paragraph{Interpretation of the Extended Shapley values.}
Each of the properties P1 to P5 is necessary to ensure the uniqueness of the valuation $\varphi$. We give some more intuition for interpreting the value Extended Shapley assigns $A$, and how the value it assigns to datapoints is different from regular Data Shapley (without taking into account $A$ and $B$):

\begin{itemize}
    \item \textbf{Value of algorithm}: In Eqn.~\ref{eqn:algorithm_value}, the coefficient is  $\frac{\card{S}! (N-\card{S})!}{(N+1)!} = \frac{1}{N+1}\cdot {N \choose \card{S}}^{-1}$. Thus intuitively, instead of simply comparing the difference in the two algorithms' overall performance $v_A(N) - v_B(N)$, Extended Shapley considers the difference  $v_A(S)-v_B(S)$, where $S$ is chosen randomly as follows: first a set size  $k$ is chosen uniformly at random from 0 to $N$, and then a random subset $S \subseteq N$ of that size is chosen. 
    \item \textbf{Value of datapoints}: 
In Eqn.~\ref{eqn:data_value}, if it weren't for the additional weight terms $w$ and $(1-w)$, the term on the left would be identical to the payoff of $i$ under Data Shapley w.r.t $A$ and the term on the right would be the payoff of $i$ under Data Shapley w.r.t $B$. Thus, one way of interpreting these two quantities is as variants of the Shapley payoff that adjusts the importance of subsets according to their size: since $w_{a,b} \to 1$ as $a \to b$, the expression on the left down-weights the importance ofthe  marginal contribution of $i$ to smaller subsets $S \subset N - i$ whereas the expression on the right down-weights the importance of the   marginal contribution of $i$ to larger subsets. 

\end{itemize}

In Appendix \ref{app:stability} we  derive a different view of the value of $A$ under Extended Shapley that is related to the notion of leave-one-out stability. We show that it can sometime give a  useful upper bound on $\varphi_A(v)$ without directly computing the expression in Eqn. \ref{eqn:algorithm_value}.

\paragraph{Efficient computation.}  As is clear from Eqn.~\ref{eqn:algorithm_value} and ~\ref{eqn:data_value}, computing the exact Extended Shapley values requires exponentially large number of computations in $N$ and therefore we extend the TMC-Shapley algorithm introduced in previous work~\cite{ghorbani2019data}. TMC-Shapley relies on two approximations: First, rearranging Eqn.~\ref{eqn:shapley} shows that the extended Shapley values are expectations of bounded random variables, which can be approximated by Monte-Carlo sampling. Secondly,  the marginal contribution of a point $i$ is close to zero when it's added to large sets: $v_A(S + i) - v_A(S) \approx 0$ for a large enough $|S|$. TMC-Shapley learns to truncate and set to 0 the terms in Eqn.~\ref{eqn:shapley} involving large sets. Details of the algorithm and its adaptation in our setting are described in Appendix C. The efficiency of this algorithm allows us to make the computation of the extended Shapley values feasible even in large data settings. For example, we show the extended Shapley values of 5 learning algorithms on the Adult Income dataset with 50k training points in Supp. Fig. 5.
%See Appendix ~\ref{sec:appendix:details:income} for an example.

\section{Applications of Extended Shapley}
\label{sec:applications}

With the definition and characterizations of Extended Shapley in place, we highlight several benefits and applications of jointly accounting for the value of the algorithm and the data points. We begin with the problem of assigning \emph{credit} to ML algorithms. We show that on three real-world disease prediction tasks, Extended Shapley consistently assigns positive value for a random forest classifier compared to logistic regression, even though the overall difference in test performance between the two models says otherwise. Next, we consider an example similar to the  one described in the introduction, where a gender-classification algorithm $A$ has significant performance disparity between different demographic groups. We show that in this case, the design of the algorithm itself is responsible for most of the disparity in prediction performance. The biased training data is also responsible for some disparity though a much smaller quantity. Finally, we consider the question of attributing responsibility in the setting of distribution shifts. We study a stylized setting in which algorithm $A$ under-performs compared to $B$ partly because there were mislabeled training instances. In this case, we show that the Shapley value of $A$ is negative ($A$ is to blame for the performance gap), but less negative then the marginal difference, highlighting that the issues with the data share some of the responsibility.

 \subsection{Understanding Extended Shapley as a new performance metric}
\label{sec:applications:progress}
    
    \paragraph{How does the algorithm's value depends on the data distribution?}
    We begin our experiments by applying Extended Shapley to a simple but illuminating setting where $A$ is a nearest neighbor classifier. We will see that Extended Shapley provides interesting insights into the way the algorithm's value depends on the data distribution and the algorithm itself. 
    For simplicity, the $N$ data points are scalars uniformly sampled from $[0, 1]$ and are  assigned a label using a binary labeling function. The labeling function divides the interval $[0, 1]$ into several sub-intervals. Points in each sub-interval are assigned one of the two labels, and the adjacent intervals are assigned the opposite label, and so on. The sub-intervals are randomly chosen such that the resulting labeling function is balanced; i.e. sub-intervals of each label will cover half of $[0, 1]$ (Fig.~\ref{fig:simple}(c)). $A$ is  a simple $k$NN (k-nearest neighbors) algorithm. The benchmark $B$ is a simple majority-vote classifier that assigns the same label to all of the test data points. The performance metric of $A$ and $B$ is the prediction accuracy on a balanced test set sampled from the same distribution. As the labeling is balanced, $v_B(S) = 0.5$ for $S \subseteq N$.  A version of this problem corresponds to a choice of a labeling function. %we will use $v_A^P(S)$ to denote the value of $A$ when trained on a subset $S$ that is sampled according to $P$.

\begin{figure*}[htb!]
        \centering
        \captionsetup{width=0.9\linewidth}
        \includegraphics[width=1.0\linewidth]{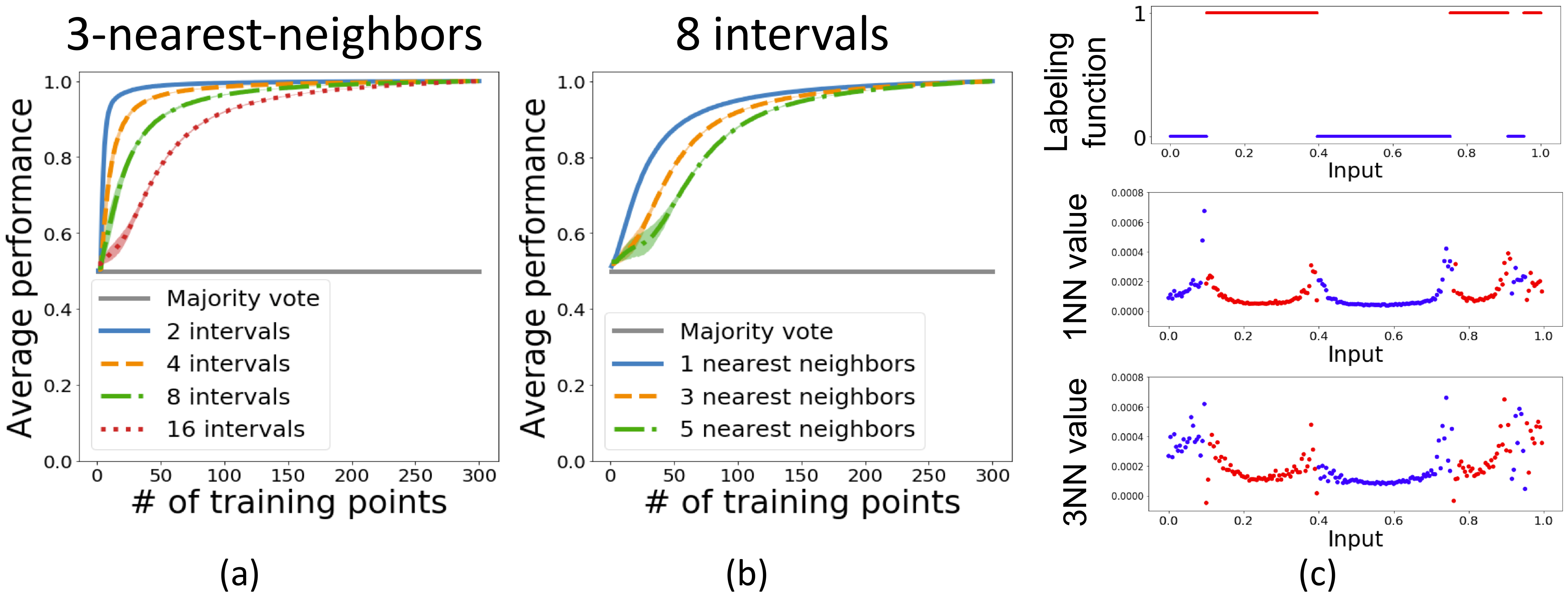}
        \caption{(a) 3NN algorithm applied to  problems of     various levels of difficulty. Using 300 training points, it achieves the same performance level on all datasets. Meanwhile, it has a smaller Shapley value  w.r.t. the Majority Vote benchmark (area between the curves and the grey line) as the dataset increases in difficulty (the \# of intervals increases). (b) $k$NN applied to the same  problem (8 intervals). Given 300 training points,  all algorithms have the same test performance, but 1NN has the highest Shapley value. (c) 1NN and 3NN applied to the same  problem (each color represents one of the labels). The Extended Shapley value of the data points tend to be higher near the interval boundaries. %The choice of algorithm affects the value of the training points: with $1$NN, as we get closer to the interval boundaries, the values increase while with $3$NN, the value first increases and then goes down. 
        %In (a) and (b), the shaded areas stand for standard deviation of $10$ different runs of the experiment. %\james{Maybe label c y axis to be 1nn data value and 3nn data value to make it clear. }
        }
        \label{fig:simple}
    \end{figure*}

    We first show how the Extended Shapley value of $A$  depends on the data distribution. In Fig.~\ref{fig:simple}(a), we apply 3-NN model on $N=300$ training points. There are 4 versions of the data distribution (2,4,8 and 16 intervals). As the number of intervals increases, the distribution becomes harder for the 3NN algorithm. The  baseline achieves 0.5 accuracy in all the versions. Each curve in Fig.~\ref{fig:simple}(a) plots the $v_A(S)$ averaged across the subsets $S$ of the same cardinality (shown on the x-axis). The Extended Shapley value of $A$ is exactly the area between the performance curve and the Majority Vote baseline. On the whole dataset $N$, the final performance $v_A(N)$ is the same across all the data versions. We see that $A$'s value decreases as the number of intervals increase. This reflects the intuition that the 3NN extracts less useful information when the neighbors have more alternating labels.

    %In particular, we will demonstrate that for different problem versions (different choice of labeling function), although the difference in performance between $A$ and $B$ ($v_A(N)-v_B(N)$) is identical across the versions, but the Shapley value of $A$ varies. Fig.~\ref{fig:simple}(a) shows the results when $A$ is a 3NN algorithm and is applied to versions with different levels of difficulty. As the number of sub-intervals in $[0, 1]$ increases (the version becomes more difficult), for a fixed size of $S$, a different $V_A(S)$ is achieved. As a result, the Extended Shapley value of $A$ against the majority vote benchmark $B$ decreases as the problem version becomes more difficult.

    Next we consider the complementary setting. We fix a version of the data distribution with 8 intervals, and apply several k-NN models for $k=1,3, 5$ (Fig.~\ref{fig:simple}(b)). Again, we see that all kNN models have the same final performance $v_A(N)$ on the full data $N$. However, the Extended Shapley value of $A$ decreases as $k$ increases. This is interesting because when the data has 8 intervals, using more neighbors is noisier for smaller training sizes.

    To conclude this example, we investigate the value of individual data points under Extended Shapley. We fix a  version with 6 intervals and apply 1NN and 3NN in this setting (Fig.~\ref{fig:simple}(c)). For each  data point, we create a data set of size $200$ by sampling $199$ data points of the same distribution and compute its value in that data set. We repeat this process 100 times and take the average value of that point.  The individual data values are plotted in the bottom two panels of Fig.~\ref{fig:simple}(c). The Extended Shapley data values are higher closer to the interval boundaries, as those points are informative. The data values for 3NN is overall slightly higher than for 1NN. The 3NN values are noisier since points in one interval (esp. near the boundary) can lead to mistakes in the adjacent interval when chosen as neighbors.
    %For the 1NN algorithm, the closer a data point is to the interval boundaries, the more value it gets while with 3NN, there is a drop in value for points very close to the boundary. \amirata{How to justify it briefly?}
    
    %Secondly, we focus on the complementary case. Assume that for a given version of the problem, we use different algorithms with identical performances ($v_A(N)-v_B(N)$ would is the same) while having different values. Fig.~\ref{fig:simple}(b) shows an example where three different algorithms ($1$NN, $3$NN, $5$NN), while having the same performance, have different Shapley values agains the majority vote benchmark.

    \paragraph{Extended Shapley value vs marginal difference in performance.} 
    As a real-world example, we compare the performance and the Extended Shapley value of two different algorithms for the problem of disease prediction. The task is to predict, given an individual's phenotypic data, whether they will be diagnosed with a certain disease in the future. We created three problems from the UK Biobank data set~\cite{sudlow2015ukb}: predicting malignant neoplasms of breast, lung and ovary (ICD10 codes C50, C34, and C56). For each problem we create a balanced training  set of 1000 individuals with half diagnosed with the disease.
    We then compare the value of a random forest (as algorithm $A$) against the logistic regression (as the benchmark $B$). 
    
    In all three problems, the Extended Shapley value of $A$, as compared to $B$, is positive. Interestingly, this is the case even though the marginal difference in performance says otherwise! For lung cancer prediction both algorithms have the same performance ($63.1\%$ and $63.2\%$) and for ovary cancer prediction logistic regression even outperforms random forest ($71.9\%$ vs $69.0\%$). This suggests that even though logistic regression achieves good final accuracy, the training data itself deserves some of the credit for this success, since the model performs poorly on smaller subsets. 
    
        \begin{figure*}[ht]
        \centering
        \captionsetup{width=1.0\linewidth}
        \includegraphics[width=\linewidth]{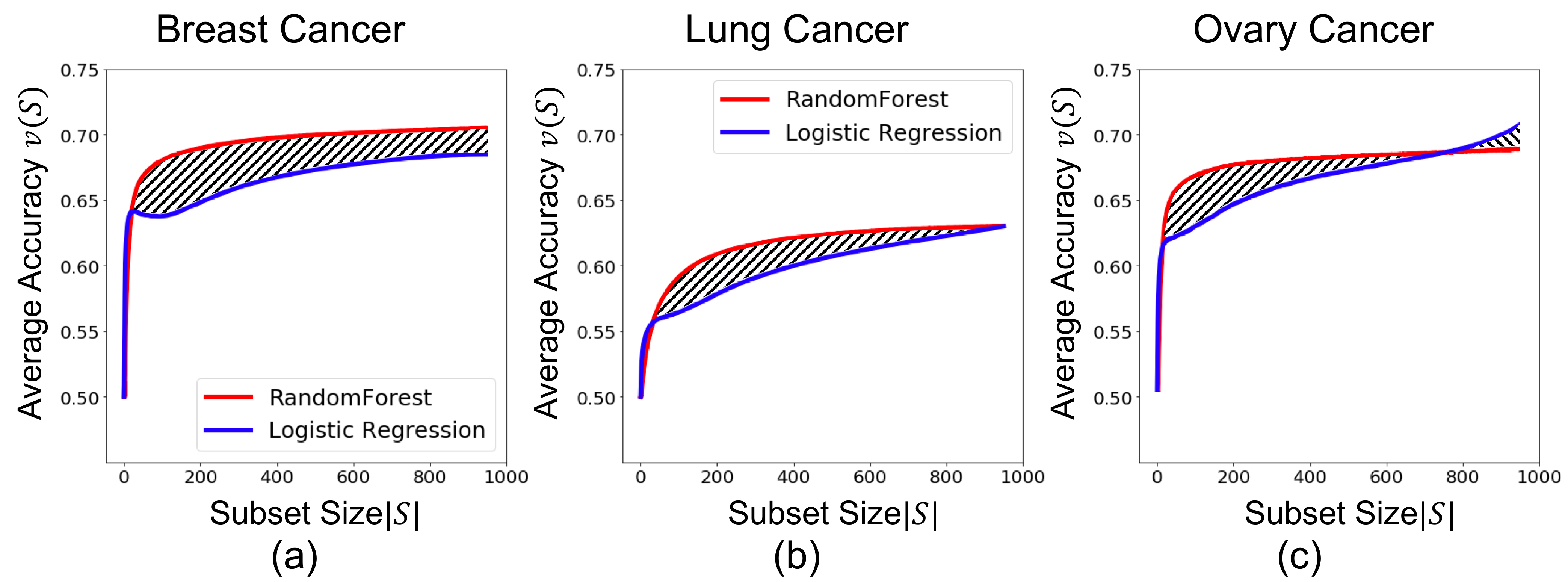}
        \caption{Following Eqn.~\ref{eqn:algorithm_value}, the shaded area corresponds to the (scaled) Shapley value of random forest (red) relative to the logistic regression baseline (blue). In all cases, the area is positive and so the value of random forest is positive, but when we consider the overall test performance (the last point of the plot), we have three different cases: (a) random forest has a better test performance, (b) both algorithms have the same test performance, (c) logistic regression has the higher test performance. 
    }
        \label{fig:ukb}

    \end{figure*}

    %The three diseases illustrate three interesting scenarios of how the Extended Shapley gives insights different from simply comparing the overall model performances   (Fig.~\ref{fig:ukb}). For breast cancer prediction,  the random forest algorithm has a higher performance ($70.5\%$ vs $68.5\%$) and a higher Extended Shapley value ($0.194$ vs $0.167$). For lung cancer prediction, both algorithms have almost the same performance ($63.1\%$ and $63.2\%$) while random forest has a higher value ($0.116$ vs $0.101$) because it learns more from smaller subset of data. Ovary cancer is a setting where the logistic regression model, while having a better performance ($71.9\%$ vs $69.0\%$), has a lower Extended Shapley value ($0.168$ vs $0.179$). This suggests that even though logistic regression achieves good final accuracy, the training data itself deserves some of its credit, since the model performs poorly on smaller subsets. 
     %A balanced test set of size 500 was used to compute the value and performance in all of the three problems. 
    
\subsection{Unfairness:  influence of data vs algorithm}
Here we analyze an example similar to the one described in the introduction.  
It has been shown that gender detection models can be have poor performance on  minorities~\cite{buolamwini2018gender}. As discussed, an important aspect here is to understand the the extent to which this unfairness stems from the data versus the algorithm.  Following work in the literature~\cite{kim2019multiaccuracy}, we use a set of 1000 images from the LFW+A\cite{wolf2011lfw1} data set that has an imbalanced representation of different subgroups ($21\%$ female, $5\%$ black) as our training data. As our performance metric, we use an \emph{equity} measure defined based on maximum performance disparity of the model:
$\mbox{equity} = 1 - 
[\max_{g \in \mbox{subgroups}}
{\mbox{Acc.(g)}} -
\min_{g \in \mbox{subgroups}}
{\mbox{Acc.(g)}}]$, with 4 subgroups: white males, white females, black males and black females. An equitable model is one that has similar accuracy on all subgroups. We use 800 images of the PPB \cite{buolamwini2018gender}, a data set designed to have equal representation of sex and different skin colors, as our test data. All images are transformed into 128 dimensional features by passing through a Inception-Resnet-V1\cite{szegedy2017inception_resnet} architecture pre-trained on the Celeb A\cite{liu2015celeba} data set (more than 200,000 face images). 

As our training algorithm $A$, we use a logistic regression model. After training it on the LFW+A image features, the model yields a test accuracy of $94\%$ on PPB data. However, the performance gap among different subgroups is large: the model is $22.9\%$ more accurate on white males compared to black females meaning that $v_A = 77.1\%$.
For the baseline algorithm $B$, we  consider a perfectly fair algorithm $B$ %that makes the correct prediction $94\%$ of the time, regardless of the input image (and therefore has $100\%$ equity). 
(so $v_B = 100\%$; e.g. $B$ can be a constant classifier). %Note that we don't actually need a real algorithm that achieves $B$; it's only serves as a fairness target to evaluate $A$.
After computing the Extended Shapley values for the data and algorithm $A$, we see that out of the $22.9\%$ equity difference between $A$ and $B$, the algorithm $A$ itself is responsible for $22.1\%$ of this disparity, while the data points are responsible for the rest. This suggests that the design of $A$ is responsible to a large part of the inequity in the performance, while the biased training set is also responsible, albeit to a much smaller extent\footnote{We highlight that in this experiment, by choosing $v$ to be equity, we have implicitly taken the perspective that equity is desirable independent of global accuracy (so $B$, despite having low accuracy, is indeed a reasonable benchmark).  It would be interesting to experiment with other performance metrics that \emph{combine} equity and global accuracy. For example, if it is the case that improving the disparity always comes at the cost of global performance, we would expect the algorithm to receive a smaller portion of the responsibility.}. %Consistent with this, in the Appendix, we show that a different design of $A$ which explicitly incorporates fairness metrics is appropriately credited by Extended Shapley for improving equity. 
Fig.~\ref{fig:values} shows the distribution of Extended Shapley values of the datapoints. Images of  men contribute to worse equity (i.e. negative Shapley value) of the final predictions. Training images of black females have the highest Shapley values. 

\begin{figure}[H]
    \centering
    \captionsetup{width=0.85\linewidth}
    \includegraphics[width=0.5\linewidth]{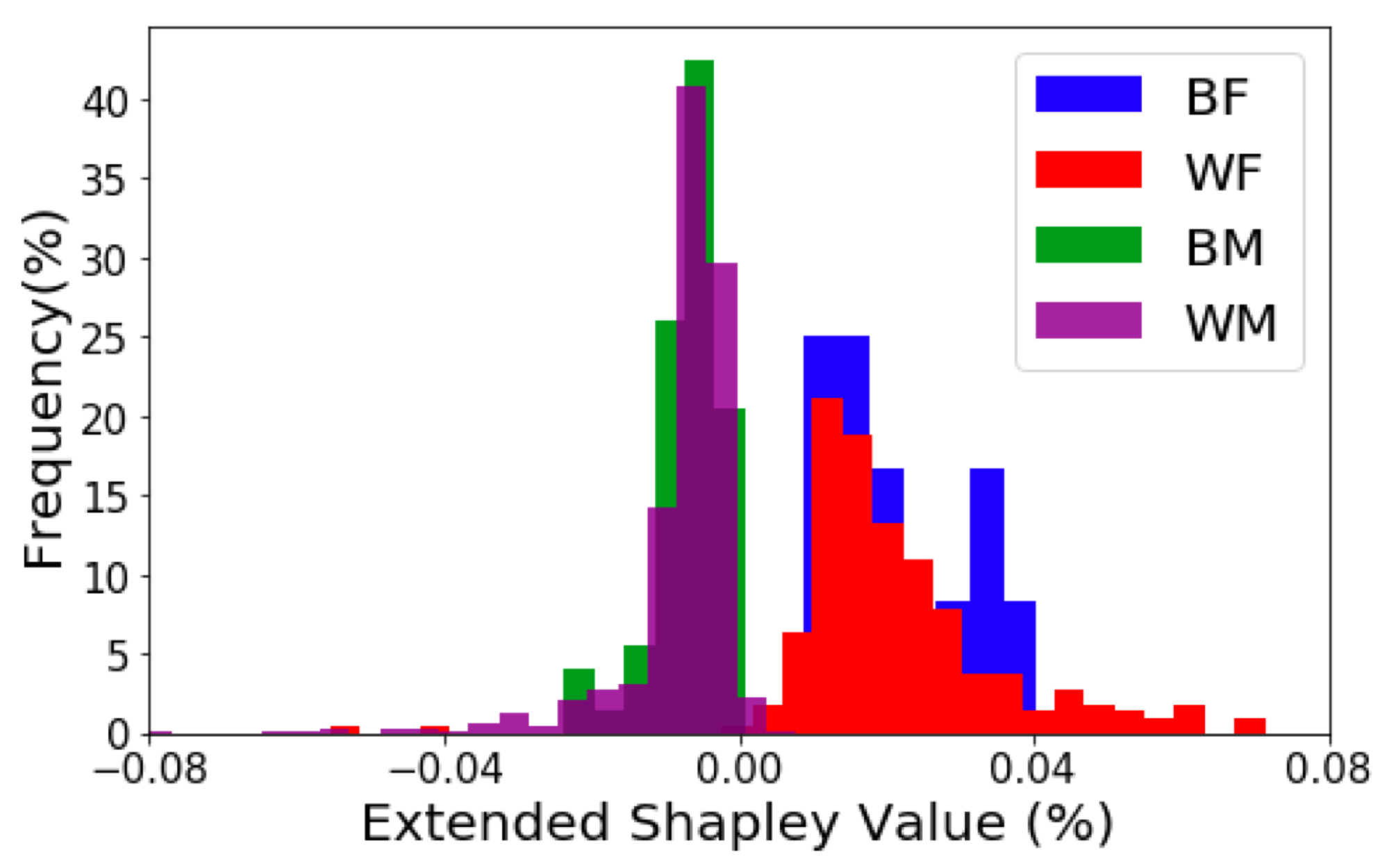}
\caption{The Extended Shapley value for the datapoints, across subgroups (in order: Black Females, White Females, Black Males, White Males). Positive value corresponds to datapoints that contribute positively to the model's equity.
    \label{fig:values}}
\end{figure}

\subsection{Accountability in the presence of distribution shifts}
\label{sec:applications:accountability}

Another common failure source for ML systems are differences between the conditions in which the model was developed and the conditions in which it is eventually deployed.  
Naturally, the choice of the learning algorithm can interact with, and potentially amplify, such differences. In this section we explore how Extended Shapley---by explicitly quantifying the value of both the ML designer and the data---can help addressing questions of accountability in these scenarios.

We instantiate this as follows. A training set is sampled from the source distribution, $D\sim \D_S$, but performance is evaluated on a test set sampled from the target distribution $D_T \sim \D_T$ (i.e., $v$ calculates some measure of accuracy on $D_T$). 
In particular, $\D_S$ and $\D_T$ are assumed to differ in that there is one sub-population whose labels are incorrectly flipped in the source distribution (see Figure \ref{fig:accountability}). As our benchmark $B$ we take the class of linear classifiers (hyperplanes). The ML designer, on the other hand chooses to work with a slightly more complex class: $A$ is the  intersection of (up to) two hyperplanes. To simplify the problem, we ignore issues of optimization and sample complexity and assume that $A$ and $B$ can solve their respective risk minimization problems exactly. We denote with $f_A, f_B$ the classifiers obtained by applying $A$ (resp., $B$) on $D$.

\begin{figure*}[htbp]
    \centering
    \captionsetup{width=0.85\linewidth}
    \includegraphics[width=0.95\columnwidth]{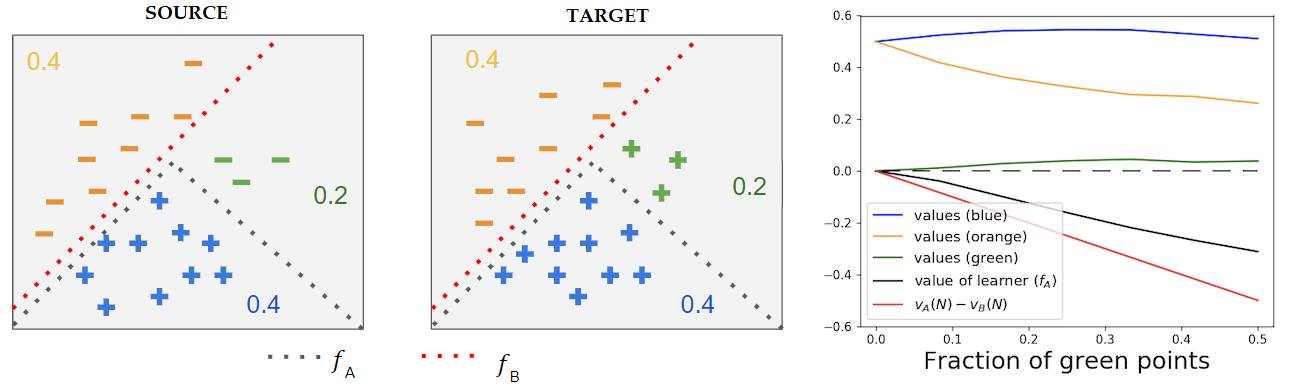}
    \caption{\textbf{Left:} The two figures on the left represent datasets sampled from the training distribution ($D$) and test distribution ($D_T$), respectively. Note that the the difference is that the label of the green points are flipped. This causes the non-linear classifier $f_A \leftarrow A(D)$ to under-perform relative to the benchmark linear classifier $f_B \leftarrow B(D)$: $v_A(D) = 0.8 < 1.0 = v_B(D)$. 
    \textbf{Right:} The value of the data and ML designer under Extended Shapley, when the fraction of green points  in the training set $p_g$ varies between $0.0$ and $0.5$. The marginal difference between the performance of $A$ and $B$ (red line) is exactly $-p_g$. Under Extended Shapley, the value of algorithm $A$ (black line) is less negative, suggesting that issue with the (green) data is partially accountable for the performance gap between $A$ and $B$. Indeed the green points have much less value than the blue and orange points.  %Additionally, as the fraction of green points increases this difference increases -- intuitively, the more mislabeled data, the less $A$ is to ``blame'' for the performance gap. 
    }

    \label{fig:accountability}
\end{figure*}

As Figure \ref{fig:accountability} shows, the mislabeling in the source distribution (left) means  $f_A$ actually \emph{under-performs} relative to $f_B$ on $D_T$ (right). Whose to blame for the reduced performance? The algorithm-centric perspective would hold the algorithm designer accountable, since the choice of an ``incorrect'' hypothesis class (non-linear classifiers) has a cost of $v_A(D) - v_B(D) = 0.8 - 1.0 = -0.2$. The data-centric perspective would hold the subset of mis-labeled points accountable, since their incorrect labels is what decreases the performance of $A$. The answer is somewhere in between: both effects (the choice of the algorithm and the distribution shift) are necessary for the performance reduction. Extended Shapley provides a principled way of apportioning the marginal decrease of performance $-0.2$ between data and algorithm.  Indeed, the right panel of Figure \ref{fig:accountability} demonstrates that Extended Shapley  assigns $A$ a value that's less negative than $-0.2$, and the gap grows larger as the fraction of green points increases. See Appendix D for a detailed description of the setup of this experiment.

\section{Conclusions}

Extended Shapley provides a principled framework to jointly quantify the contribution of
the learning algorithm and training data to the overall models' performance. It's  a step in formalizing the notion of ML accountability, which is increasingly important especially as ML becomes more wide-spread in mission critical applications. 
The strong axiomatic foundation of Extended Shapley guarantees equitable treatment of data and algorithm.

A goal of our paper is to highlight the importance of jointly quantifying the accountability of the learning algorithm and the training data. This question is becoming increasingly critical as we seek to ensure that our AI systems are broadly beneficial. As a first step toward this ambitious goal, we propose Extended Shapley as one theoretically motivated approach. We recognize that there are potential limitations to the modeling assumptions behind Extended Shapley, and it would be interesting to explore alternatives. We hope this work provides one formal way to think about joint algorithm and data accountability and opens new research questions.  

There are several natural directions for future work. First, it would useful to extend our approach to multiple algorithms and multiple benchmarks. Second, we focused on using Extended Shapley as a diagnostics tool, but it would be interesting to investigate
the degree to which these insights can also be used to address issues at test time.
 Third, using Extended Shapley to debug other real-world failure cases could be interesting and instructive. Finally, it would be interesting to study other approaches to the joint data and algorithm valuation problem, for example using other solution concepts \cite{Yan2020IfYL}.

\newpage
\begin{appendix}
\section{Missing proofs}
\label{app:proofs}

\subsection{Proof of Proposition \ref{prop:unique_p1p5}}
\label{app:proofs:es}

The proof of the proposition will consist of two parts: first we define a valuation scheme and prove that it satisfies P1-P5; then we will show that this valuation scheme takes the form in the proposition statement. Define $\varphi : G^2 \to \R^{n+1}$ as follows:

\begin{equation*}
    \varphi(v) = \varphi(v_A, v_B) =  \phi^{SHAP}(\tilde{v}_{A,B})
\end{equation*}

where $\phi^{SHAP}$ is the regular Data Shapley and $\tilde{v}_{A,B}:2^{N+1}\to\R:$ is defined as
\begin{equation}
\label{eqn:v_tilde}
    \tilde{v}_{A,B}(S)=\begin{cases}
v_{A}(S-\{n+1\}) & n+1\in S\\
v_{B}(S) & n+1\notin S
\end{cases}
\end{equation}

We will use the fact that $\phi^{SHAP}$ satisfies properties S1-S4 to prove that $\varphi : G^2 \to \R^{n+1}$  satisfies properties P1-P5. 

\textbf{P1}. First, assume $i\in N$ is a null player in $v_A$ and $v_B$. We claim that this implies $i$ is a null player in the game $\tilde{v}_{A,B}$. To show this, we must prove that for every $S\subset(N+1)-i$, $\tilde{v}_{A,B}(S+i)=\tilde{v}_{A,B}(S)$. Indeed: if the subset includes $n+1$, the requirement is equivalent to $v_{A}(S+i)=v_{A}(S)$; if it doesn't, the  requirement is equivalent to $v_{B}(S+i)=v_{B}(S)$. In either case the statement is true since $i$ was assumed to be a null player in both $v_A$ and $v_B$. Thus from S1 we have that $\varphi_i(v)=\phi_i(\tilde{v}_{A,B})=0$.
    Finally, consider $n+1$. Note that for any $S\subset N$, the requirement that $\tilde{v}(S+(n+1))=\tilde{v}(S)$ is identical, by the definition of $\tilde{v}_{A,B}$, to $v_{A}(S)=v_{B}(S)$. Thus the assumption of P1 means $n+1$ is a null player in the game $\tilde{v}_{A,B}$. From S1 we therefore have that $\varphi_{n+1}(v)=\phi_{n+1}(\tilde{v}_{A,B})=0$, as required.

\textbf{P2}. Suppose $i,j$ are identical in both $v_{A}$ and $v_{B}$; we will show that this means they are identical in the game $\tilde{v}_{A,B}$. Consider subsets of $N+1-\{i,j\}$. Suppose a subset includes $n+1$, then $\tilde{v}(S+i)=\tilde{v}(S+j)$ is equivalent to $v_{A}(S+i)=v_{A}(S+j)$ and the later is true since $i,j$ are identical in $v_{A}$. Similarly if the subset doesn't include $n+1$, this is equivalent to $v_{B}(S+i)=v_{B}(S+j)$ which is true since $i,j$ are identical in $v_{B}$. Thus from S2, we have that $\varphi_i(v)=\phi_i(\tilde{v})=\phi_j(\tilde{v})=\varphi_j(v)$.    

\textbf{P3}.  Let $v_{A1}, v_{A2}, v_{B1}, v_{B2}$ be any four elements in $G$. We need to prove that $\varphi(v_{A1} + v_{A2}, v_{B1} + v_{B2})$ is the same as $\varphi(v_{A1}, v_{B1}) + \varphi(v_{A2}, v_{B2})$. Let $\tilde{v}$ denote the application of Equation (\ref{eqn:v_tilde}) w.r.t $v_{A1} + v_{A2}$ and $v_{B1} + v_{B2}$:

\begin{equation*}
    \tilde{v}(S)=\begin{cases}
(v_{A1}+v_{A2})(S-\{n+1\}) & n+1\in S\\
(v_{B1}+v_{B2})(S) & n+1\notin S
\end{cases}
\end{equation*}

By direction calculation, we see that $\tilde{v} \equiv \tilde{v}_{A1,B1}+\tilde{v}_{A2,B2}$. Thus, the linearity property for the regular Data Shapley guarantees that $\phi^{SHAP}(\tilde{v})=\phi^{SHAP}(\tilde{v}_{A1,B1})+\phi^{SHAP}(\tilde{v}_{A2,B2})$, which is, by definition, the same as the required.

%We will prove that $\varphi$ is linear in its first component; an identical argument can be used to show linearity in the second. Let $v_{A1}, v_{A2}, v_{B}$ be any three games in $G$. Note that $\tilde{v}$ satisfies: $\tilde{v}_{A1+A2,B}=\tilde{v}_{A1,B}+\tilde{v}_{A2,B}$. Thus from S3, $\phi(\tilde{v}_{A1+A2,B})=\phi(\tilde{v}_{A1,B})+\phi(\tilde{v}_{A1,B})$. We therefore have that $\varphi(v_{A1}+v_{A2}, v_B) = \varphi(v_{A1}, v_B) + \varphi(v_{A2}, v_B)$, as required.
    
\textbf{P4}. Indeed, $\sum_{i\in N+1}\varphi_{i}(v)=\sum_{i\in N+1}\phi_{i}(\tilde{v}_{A,B}) =
        \tilde{v}_{A,B}(N+1)=v_{A}(N)$,
    where the second transition is from S4 and the third is directly by the definition of $\tilde{v}_{A,B}$.

\textbf{P5}. Note that by the definition of $\tilde{v}$, $v_B(S+i)=v_A(S)$ is equivalent to $\tilde{v}(S+i)=\tilde{v}(S+(n+1))$, for every $S\subseteq N-i$. Then the assumption in this property implies that $n+1$ and $i$ are identical; from S2, we have that  $\varphi_i(v)=\phi_i(\tilde{v})=\phi_{(n+1)}(\tilde{v})=\varphi_A(v)$, as required.

We now turn to prove that the value assigned to the algorithm and to the datapoints follow the expressions in Equations \ref{eqn:algorithm_value} and \ref{eqn:data_value}. For the value of the algorithm, this follows directly by applying the definition of the Shapley value in the $n+1$ player game:

\begin{align*}
    \varphi_A(v) &= \phi^{SHAP}_{n+1}(\tilde{v}_{A,B}) \\
    &= \sum_{S\subseteq N}{\tau_{S, N+1} \cdot \sbr{\tilde{v}_{A,B}(S+(n+1)) - \tilde{v}_{A,B}(S)}} \\
    &= \sum_{S\subseteq N}{\tau_{S, N+1} \cdot \sbr{v_A(S) - v_B(S)} }
\end{align*}

We now turn to proving Equation \ref{eqn:data_value}. First, to simplify notation, we will use $\Delta v(S,i)$ to denote the marginal difference of $v$ evaluated on $S+i$ versus $S$:
\begin{equation*}
    \Delta v(S,i) = v(S+i) - v(S)
\end{equation*}
By the definition of Extended Shapley, 
\begin{equation*}
    \varphi_i(v) = \sum_{S \subseteq N+1 -i}\tau_{S,N+1}\cdot \Delta \tilde{v}_{A,B}(S, i)
\end{equation*}

We now split this sum into two terms, those subsets that include $n+1$ and those that don't:

\begin{equation*}
\sum_{\tiny \substack{S \subseteq N+1 -i \\ n+1 \in S}}\tau_{S,N+1} \Delta \tilde{v}_{A,B}(S, i)   +  \sum_{\tiny \substack{S \subseteq N+1-i \\ n+1 \notin S}}\tau_{S,N+1} \Delta \tilde{v}_{A,B}(S, i)
\end{equation*}

and consider each term separately. For the term on the right, note that iterating over subsets of $N+1-i$ that don't include $n+1$ is equivalent to iterating over subsets of $N-i$, and that in this case $\tilde{v}$ is evaluated as $v_B$;  %Additionally, recall the definition of $\tau$ as $\tau_{S,N} = \frac{\card{S}!(\card{N}-\card{S}-1)!}{\card{N}!}$. 
we can therefore re-write the term on the right as:

\begin{equation}
    \sum_{S\subseteq N-i}{\tau_{S, N+1}\cdot \Delta v_B(S, i)}
\end{equation}

%\begin{align*}
%  &  \sum_{S\subset N-i}{\frac{\card{S}!\br{N-\card{S}}!}{\br{N+1}!}\cdot \Delta v_B(S, i)} = \\ & 
%    \sum_{S\subset N-i}{\frac{N-\card{S}}{N+1} \cdot \frac{\card{S}!\br{N-\card{S}-1}!}{N!}\cdot \Delta v_B(S, i)}  = \\ &
%     \sum_{S\subset N-i}{\br{1-w_{S,N}}\tau_{S,N}\cdot \Delta v_B(S, i)}
%\end{align*}

For the  term on the left:
\begin{align*}
     \sum_{\substack{S \subseteq N+1 -i \\ n+1 \in S}}{\tau_{S, N+1}\cdot \Delta \tilde{v}_{A,B}(S,i)} =
     \sum_{\substack{S \subseteq N-i \\S' \leftarrow S +(n+1)}}{\tau_{S', N+1}\cdot \Delta \tilde{v}_{A,B}(S',i)} =
    \sum_{S\subseteq N-i}{\tau_{S+1, N+1}\cdot \Delta v_A(S,i)}
\end{align*}

Together, we have that that $\varphi_i(v)$ can be written as

\begin{equation*}
      \sum_{S\subseteq N-i}{\tau_{S+1, N+1}\cdot \Delta v_A(S,i)} + \sum_{S\subseteq N-i}{\tau_{S, N+1}\cdot \Delta v_B(S, i)}
\end{equation*}

The proof can be concluded by observing that $\tau_{S, N+1} = (1-w_{S,N})\cdot \tau_{S,N}$ and $\tau_{S+1, N+1} = w_{S,N} \cdot \tau_{S,N}$.

\section{An alternative view of the algorithm's value in Extended Shapley}
\label{app:stability}

We can provide a different view of the ML designer's value under Extended Shapley, that is related to the notion of leave-one-out stability.

\begin{lemma}
\label{lemma:shapley_learner2}
For a game $v\in G$, define $\overline{v} \in G$ as follows: $ \overline{v}(S) = \sum_{i \in N-S}{\sbr{v(S+i)-v(S)}}$. Then,
%\vspace{-0.75em}
the value of the ML designer under Extended Shapley can also be written as
\begin{equation*}
    \varphi_A(v) = \sum_{S \subseteq N} {\frac{S! (N-S)!}{(N+1)!} \sbr{ \overline{v}_A(S) - \overline{v}_B(S)}}
\end{equation*}
\end{lemma}

\begin{proof}
First, we note that by combining the efficiency axiom with the charectarization of the previous proposition, we can also write the ML designer's value as

\begin{multline*}
     \varphi_{A}(v) = \sum_{i\in N} \sum_{S\subset N-i}(1-w_{S,N})\cdot\tau_{S,N}\cdot\sbr{v_{A}(S+i)-v_A(S)}  -  \sum_{i\in N} \sum_{S\subset N-i}(1-w_{S,N})\cdot\tau_{S,N}\cdot\sbr{v_{B}(S+i)-v_B(S)}
\end{multline*}

To see this, first note that from the efficiency property of Extended Shapley (P4), $\sum_{i \in N}{\varphi_i(v)} + \varphi_{A}(v) = v_A(N)$. Additionally, from the efficiency property of the standard Data Shapley (S4) w.r.t $A$, $v_A(N) = \sum_{i\in N}{\varphi_i(v_A)}$. Thus, $\varphi_{A}(v) = \sum_{ i\in N}{\varphi_i(v_A)} - \sum_{i \in N}{\varphi_i(v)} $. By the definition of the Shapley value (Equation \ref{eqn:shapley}), the term on the left is $\sum_{i\in N}\sum_{S\subset N - i}\tau_{S,N}\cdot \Delta v_A(S, i)$. Now, by substituting  $\varphi_i(v)$ with the expression for the value of the datapoints (Equation (\ref{eqn:data_value}) and re-arranging, we get exactly the above.

Next, we note that in general the sum $\sum_{i\in N}\sum_{S\subseteq N-i}$ is equivalent to the sum $\sum_{S\subseteq N}\sum_{i\in N-S}$. We can therefore write:
\begin{align*}
& \sum_{i\in N}\sum_{S\subseteq N-i}(1-w_{S,N})\tau_{S,N} \Delta v(S,i) \\	&=\sum_{S\subseteq N}\sum_{i\in N-S}(1-w_{S,N})\tau_{S,N} \Delta v(S,i) \\
	&=\sum_{S\subseteq N}(1-w_{S,N})\tau_{S,N}\sum_{i\in N-S} \Delta v(S,i)  \\
	&= \sum_{S\subseteq N}(1-w_{S,N})\tau_{S,N} \cdot \overline{v}(S)
\end{align*}

Combining these two facts, we have:
\begin{equation*}
    \varphi_A(v)= \sum_{S\subseteq N}(1-w_{S,N})\tau_{S,N} \sbr{\overline{v}_A(S)-\overline{v}_B(S)}
\end{equation*}

The proof can be concluded by noting that $(1-w_{S,N})\cdot \tau_{S,N} = \tau_{S,N+1} = {\frac{S! (N-S)!}{(N+1)!}}$.
\end{proof}

In some cases, the view of $A$'s value in Lemma \ref{lemma:shapley_learner2} can provide an upper bound on $\varphi_A(v)$ without directly computing the expression in Eqn.~\ref{eqn:algorithm_value}.

\begin{lemma}
\label{lemma:shapley_upperbound}
If $A,B$ are $\gamma$-leave-one-out-stable w.r.t the performance metric $v$, i.e., for both $A$ and $B$ , $\max_{S\subseteq N, i \in N}|v(S+i) - v(S)| \leq \gamma$, then $\varphi_A(v) \leq n\cdot \gamma/2$.
\end{lemma}

\begin{proof}

The claim follows by proving that if $v\in G$ is $\gamma$-stable in the sense defined in the statement of the lemma, then 
\begin{equation*}
    \sum_{S\subseteq N}\tau_{S,N+1}\cdot v(S)\leq\frac{n}{2}\cdot\gamma
\end{equation*}

To prove this, we employ the second view of the ML designer's value from the previous lemma; in particular, it guarantees that 
\begin{equation*}
    \sum_{S\subseteq N}\tau_{S,N+1}\cdot v(S) =\sum_{S\subseteq N}\tau_{S,N+1}\overline{v}(S)
\end{equation*}

Which we can simplify as follows:
\begin{align*}
    & \sum_{S\subseteq N}\tau_{S,N+1}\overline{v}(S)\\ &=\sum_{S\subseteq N}\tau_{S,N+1}\sum_{i\in N-S}[v(S+i)-v(S)]\\ &\leq\sum_{S\subseteq N}\tau_{S,N+1}\sum_{i\in N-S}\gamma\\&=\sum_{S\subseteq N}\tau_{S,N+1}\left|N-S\right|\cdot\gamma\\&=\gamma\cdot\sum_{S\subseteq N}\tau_{S,N+1}\left|N-S\right|\\&=\gamma\sum_{k=0}^{n}{n \choose k}\tau_{S,N+1}(n-k)\\&=\gamma\sum_{k=0}^{n}\frac{n-k}{n+1}\\&=\gamma\frac{n(n+1)}{2(n+1)}\\&=\frac{n}{2}\cdot\gamma
\end{align*}

Together, we can conclude the required.
\end{proof}

The following example demonstrates that the upper-bound is tight.

\begin{exmp}
Suppose that $A$ and $B$ are given by $v_A(S)=\card{S}$ and $v_B(S) \equiv 0$. In this case, the value of $A$ under Extended Shapley is exactly $\frac{n}{2}$:

\begin{align*}
    \varphi_{A}(v)&=\sum_{S\subseteq N}\tau_{S,N+1}\cdot[v_{A}(S)-v_{B}(S)] \\&=\sum_{S\subseteq N}\tau_{S,N+1}\cdot v_{A}(S)=\sum_{k=0}^{n}{n \choose k}\tau_{S,N+1}\cdot k \\&=\sum_{k=0}^{n}\frac{k}{n+1} = \frac{n}{2}
\end{align*}

Note that this is exactly the upper-bound from Lemma \ref{lemma:shapley_upperbound}, since $A$ is $\gamma=1$ stable and $B$ is $\gamma=0$ stable.

\end{exmp}

\section{Algorithmic Approximations}
\label{appendix:alg}
\input{appendix_alg.tex}

\clearpage
\section{Experiments}

\label{appendix:exp_details}
\input{experiments_details.tex}

\end{appendix}

\newpage

\bibliographystyle{apalike}
\bibliography{refs}

\end{document}

%% file: abstract.tex
\begin{abstract}
A learning algorithm $A$ trained on a dataset $D$ is revealed to have poor performance on some subpopulation at test time. Where should the responsibility for this lay? It can be argued that the data is responsible, if for example training $A$ on a more representative dataset $D'$ would have improved the performance. But it can similarly be argued that $A$ itself is at fault, if  training a different variant  $A'$  on the same dataset $D$ would have improved performance. As ML becomes widespread and such failure cases more common, these types of questions are proving to be far from hypothetical. With this motivation in mind, in this work we provide a rigorous formulation of the joint credit assignment problem between a learning algorithm $A$ and a dataset $D$. We propose Extended Shapley as a principled framework for this problem, and experiment empirically with how it can be used to address questions of ML accountability.

\end{abstract}

%% file: appendix_alg.tex
The TMC-Shapley was introduced by~\cite{ghorbani2019data} as a method for efficient approximation of the data Shapley values. The algorithm relies on two approximation schemes:
\emph{(i)}  Rearranging Eq.~\ref{eqn:shapley} we can write the following: $\phi_i = \frac{1}{n!} \mathbb{E}_{\pi \sim \Pi} [v(S_\pi^i \cup \{i\}) - v(S_\pi^i)]$, 
where $\Pi$ is the uniform distribution over all $n!$ permutations, $S_{\pi}^{i}$ is the set of data points coming before datum $i$ in permutation $\pi$ ($S_{\pi}^{i}= \emptyset$ if $i$ is the first element) which means computing the Data Shapley value reduces to computing the expectation of a bounded random variable. As a result, monte-carlo sampling could be used to have an unbiased estimation.
\emph{(ii)} For machine learning models, the performance $v(S \subseteq N)$ tends to saturate as $|S|$ increases, meaning, for a large enough $|S|$, the marginal contribution $v(S + i) - v(S)$ is small and could be approximated by zero for a given tolerance of bias in our approximations. The details of the original TMC-Shapley algorithm are described in Alg.~\ref{alg:TMC-Shapley}. Our adaptation of this method to the setting of extended Shapley is described by Alg.~\ref{alg:TMC-EShapley}.

\begin{algorithm}[h]
  \caption{\textbf{Truncated Monte Carlo Shapley}}
  \label{alg:TMC-Shapley}
      \begin{algorithmic}
        \STATE {\bfseries Input:}  Train data $N = \{1,\dots,n\}$, learning algorithm A, performance metric $v$
        \STATE {\bfseries Output:} Data Shapley values: $\phi_1,\dots,\phi_n$
        \vspace{1mm}
        \STATE Initialize $\phi_i=0$ for $i=1,\dots,n$ and $t=0$
        \WHILE{Data Shapley values have not converged}
            \STATE $t \leftarrow t+1$
            \STATE $\pi^{t}$: Random permutation of $\{1,\ldots,n\}$
            \STATE $v^t_0 \leftarrow v_A(\emptyset)$ 
            \FOR{$j \in \{1,\ldots,n\}$}
                \IF{$|v_A(N) - v^t_{j-1}| < \mbox{Performance Tolerance}$}
                    \STATE $v^t_{j} = v^t_{j-1}$
                \ELSE
                    \vspace{0.5mm}
                    \STATE $v^t_j \leftarrow  v_A(\{\pi^{t}[1],\dots,\pi^{t}[j]\})$
                    \vspace{0.5mm}
                \ENDIF
                \STATE $\phi_{\pi^t[j]} \leftarrow \frac{t-1}{t}\phi_{\pi^{t-1}[j]} + \frac{1}{t}(v^t_{j} - v^t_{j-1}) $  
                \vspace{1mm}
             \ENDFOR
        \ENDWHILE
        \vspace{1mm}
      \end{algorithmic}
\end{algorithm}

\begin{algorithm}[h]
  \caption{\textbf{Truncated Monte Carlo Extended Shapley}}
  \label{alg:TMC-EShapley}
      \begin{algorithmic}
        \STATE {\bfseries Input:}  Train data $N = \{1,\dots,n\}$, learning algorithms A and B, performance metric $v$
        \STATE {\bfseries Output:} Data Shapley values: $\varphi_1,\dots,\varphi_n$
        \vspace{1mm}
        \STATE Initialize $\varphi_i=0$ for $i=1,\dots,n$ and $t=0$ and $\varphi_A = 0$
        \WHILE{Data Shapley values have not converged}
            \STATE $t \leftarrow t+1$
            \STATE $\pi^{t}$: Random permutation of $\{1,\ldots,n\}$
            \STATE $v^t_{A, 0} \leftarrow v_A(\emptyset)$,   $v^t_{B, 0} \leftarrow v_B(\emptyset)$
            \FOR{$j \in \{1,\ldots,n\}$}
                \IF{$|v_A(N) - v^t_{A,j-1}| < \mbox{Performance Tolerance}$}
                    \STATE $v^t_{A,j} = v^t_{A,j-1}$
                \ELSE
                    \vspace{0.5mm}
                    \STATE $v^t_{A, j} \leftarrow  v_A(\{\pi^{t}[1],\dots,\pi^{t}[j]\})$
                    \vspace{0.5mm}
                \ENDIF
                \vspace{0.5mm}
                \IF{$|v_B(N) - v^t_{B,j-1}| < \mbox{Performance Tolerance}$}
                    \STATE $v^t_{B,j} = v^t_{B,j-1}$
                \ELSE
                    \vspace{0.5mm}
                    \STATE $v^t_{B,j} \leftarrow  v_B(\{\pi^{t}[1],\dots,\pi^{t}[j]\})$
                    \vspace{0.5mm}
                \ENDIF

                 \vspace{0.5mm}
                 \STATE
                $\varphi_{\pi^t[j]}  \leftarrow \frac{t-1}{t}\varphi_{\pi^{t-1}[j]} + \frac{w_{t-1,n}}{t}(v^t_{A,j} - v^t_{A,j-1}) + \frac{1 - w_{t-1,n}}{t}(v^t_{B,j} - v^t_{B,j-1}) $  
                
                \vspace{0.5mm}
                \STATE
                $\varphi_A \leftarrow \frac{t-1}{t}\varphi_{A} + \frac{1}{t} (v^t_{A,j} - v^t_{A,j-1}) - (v^t_{B,j} - v^t_{B,j-1})$
             \vspace{0.5mm} 
             \ENDFOR
             
        $\varphi_A \leftarrow \frac{1}{n+1} \varphi_A$    
        \ENDWHILE
        \vspace{1mm}
      \end{algorithmic}
\end{algorithm}

%% file: experiments_details.tex
\subsection{Computing extended-Shapley for a large dataset}
\label{sec:appendix:details:income}
In this experiment, we use the adult income binary prediction task from the UCI repository of machine learning datasets.~\cite{UCI} The dataset has more than 48000 data points. We use 45000 data points as our training set $N$ and the rest as the test set. Using Alg.~\ref{alg:TMC-EShapley}, we compute the extended Shapley values for $5$ different learning algorithms shown in Fig.~\ref{fig:income}. As it is clear, the Adaboost algorithm has a positive extended-Shapley value versus every other algorithm.

\begin{figure*}[ht]
    \centering
    \captionsetup{width=.9\linewidth}
    \includegraphics[width=0.55\columnwidth]{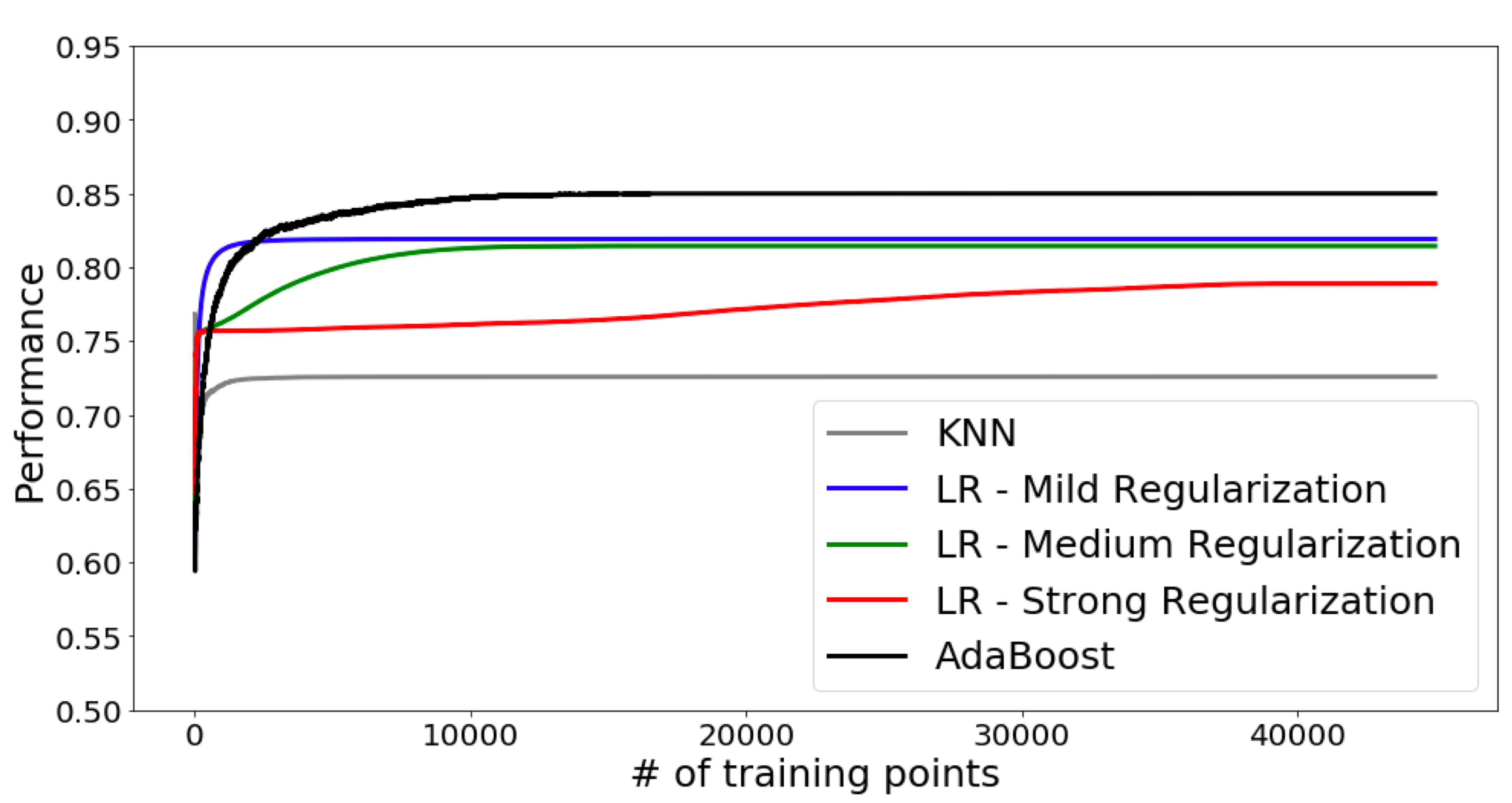}
    \caption{Performance plots for $5$ different algorithms are computed for the adult income data set.}
    \label{fig:income}
\end{figure*}

\subsection{Details of the experiments in Section  \ref{sec:applications:accountability}}
\label{sec:appendix:details:accountability}
Every point on the x-axis of the right hand side of Figure \ref{fig:accountability} corresponds to a problem instance as illustrated in Figure \ref{fig:source_target_details}. 

\begin{figure*}[h]
    \centering
    \captionsetup{width=.9\linewidth}
    \includegraphics[width=0.7\columnwidth]{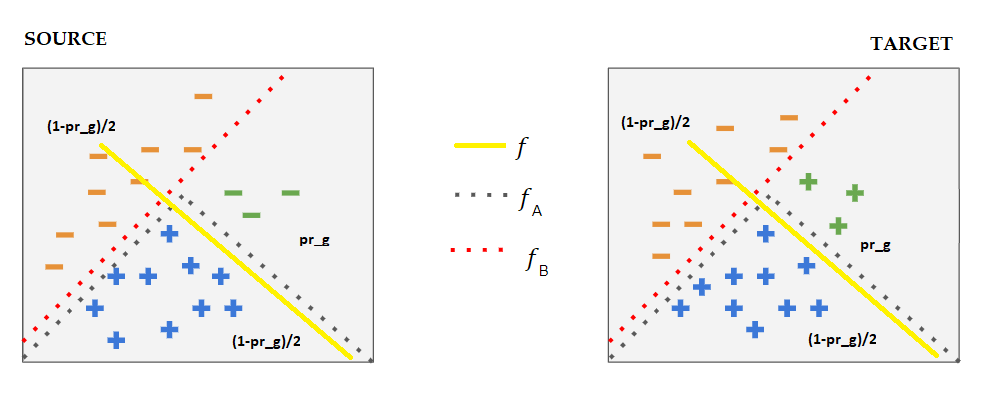}
    \caption{A problem instance is defined by $p_g$, the fractional mass of the green points. They are mis-labeled: in the source distribution their label is $-$, yet  in the target distribution it is $+$. The other two groups are identical between the source and target distribution, and their fractional mass are given by $p_y = p_b = (1-p_g)/2$. }
    \label{fig:source_target_details}
\end{figure*}

We now describe how we compute the value of the algorithm $A$ for each problem instance. It is sufficient to define $v_A$ and $v_B$. For simplicity, we assume that the classifier obtained by training algorithms $A$ and $B$ on a subset $S \subseteq N$ is only a function of the colors of the points in $S$; see Table \ref{table:accountabiliy:classifiers} for a full specification of the assumed behaviours of $A$ and $B$ and the implication for the values of different subsets.

\begin{table}[h]
\centering
\captionsetup{width=.85\linewidth}

\begin{tabular}{@{}lcccc@{}}
\toprule
 colors($S$) & $A(S)$  & $B(S)$ \\
\midrule
y, b, g & $f_A$ &  $f_B$ \\
y, b  &   $f_B$ &   $f_B$ \\
b, g  &   $f$ &   $f$ \\
y, g  & $h^-$ &   $h^-$ \\
y  & $h^-$ &   $h^-$ \\
g  & $h^-$ &   $h^-$ \\
b  & $h^+$ &   $h^+$ 
\\ \bottomrule
\end{tabular}
\quad
\begin{tabular}{@{}lcccc@{}}
\toprule
 colors($S$) & $v_A(S)$  & $v_B(S)$ \\
\midrule
y, b, g & $1-p_g$ &  $1$ \\
y, b  &   $1$ &   $1$ \\
b, g  &   $p_b + \frac{1}{2}p_y$ &   $p_b + \frac{1}{2}p_y$ \\
y, g  & $p_y$ &   $p_y$ \\
y  & $p_y$ &   $p_y$ \\
g  & $p_y$ &  $p_y$ \\
b  & $p_b+p_g$ &   $p_b+p_g$
\\ \bottomrule
\end{tabular}
\end{table}
\begin{table}[H]

\captionsetup{width=.85\linewidth}
    \centering

        % exp, before: 0.06, after 0

    \caption{\textbf{Left, bottom to top:}  when there are only blue points, the entire training set is labeled positively and the output is $h^+$, the classifier that labels everything as positive. Similarly, when there are only either yellow or green points, the output is $h^-$.  When $S$ consists of blue and green points, both models output the linear classifier $f$ (yellow line). When $S$ consists of yellow and blue points, both models output the linear classifier $f_B$. Finally, when $S$ consists of all three colors, the linear model $B$ outputs $f_B$ whereas $A$ outputs the non-linear classifier $f_A$. \textbf{Right, bottom to top:}  The performance of $h^+$ is $p_g+p_b$ as it correctly classifies both the blue and the green points; similarly, the performance of $h^-$ is $p_y$ as it correctly classifies only the yellow points. The performance of $f$ is $p_b + \frac{1}{2}p_y$, since on $D_T$ it correctly classifies the blue points and half of the yellow points. The performance of $f_A$ is $1-p_g$, since it only errs on the green points. Finally, the performance of $f_B$ is 1.0 since it correctly classifies all the points. }
    \label{table:accountabiliy:classifiers}
\end{table}

%The performance of the classifiers above depends on the problem instance, which is defined by $(p_y, p_b, p_g)$, the fractional mass of the yellow, blue and green subgroups. The following table summarizes $v_A, v_B \in G$ as a function of these probabilities: